\documentclass[10pt, conference, compsocconf]{IEEEtran}

\IEEEoverridecommandlockouts

\usepackage{graphicx}
\usepackage{amsmath,amssymb} 
\usepackage{epstopdf}

\usepackage{caption}
\usepackage{subcaption}
\usepackage{epsfig}

\usepackage{algorithm}
\usepackage{algpseudocode}

\def \xx {{\bf x}}
\def \yy {{\bf y}}
\def \ee {{\bf e}}

\def \ww {{\bf w}}
\def \XX {{\bf X}}
\def \YY {{\bf Y}}

\newtheorem{lemma}{Lemma}

\begin{document}

\title{Accelerating Deep Learning with Shrinkage and Recall}

\author{
\IEEEauthorblockN{
Shuai Zheng\IEEEauthorrefmark{1}, Abhinav Vishnu\IEEEauthorrefmark{2} and Chris Ding\IEEEauthorrefmark{1}
}
\IEEEauthorblockA{
\IEEEauthorrefmark{2}Pacific Northwest National Laboratory, Richland, WA, USA\\
abhinav.vishnu@pnnl.gov
}
\IEEEauthorblockA{
\IEEEauthorrefmark{1}Department of Computer Science and Engineering,\\ University of Texas at Arlington, TX, USA\\
zhengs123@gmail.com, chqding@uta.edu
}
\thanks{This work was conducted while the first author was doing internship
at Pacific Northwest National Laboratory, Richland, WA, USA.}
}

\maketitle

\begin{abstract}
Deep Learning is a very powerful machine learning model. Deep Learning trains a large number of parameters for multiple layers and is very slow when data is in large scale and the architecture size is large. Inspired from the shrinking technique used in accelerating computation of Support Vector Machines (SVM) algorithm and screening technique used in LASSO, we propose a shrinking Deep Learning with recall (sDLr) approach to speed up deep learning computation. We experiment shrinking Deep Learning with recall (sDLr) using Deep Neural Network (DNN), Deep Belief Network (DBN) and Convolution Neural Network (CNN) on 4 data sets. Results show that the speedup using shrinking Deep Learning with recall (sDLr) can reach more than $2.0$ while still giving competitive classification performance.
\end{abstract}


\IEEEpeerreviewmaketitle

\section{Introduction}
Deep Learning \cite{hinton2006reducing} has become a powerful machine learning model. It differs from traditional machine learning approaches in the following aspects: Firstly, Deep Learning contains multiple non-linear hidden layers and can learn very complicated relationships between inputs and outputs. Deep architectures using multiple layers outperform shadow models \cite{bengio2009learning}. Secondly, there is no need to extract human design features \cite{krizhevsky2012imagenet}, which can reduce the dependence of the quality of human extracted features. We mainly study three Deep Learning models in this work: Deep Neural Networks (DNN), Deep Belief Network (DBN) and  Convolution Neural Network (CNN).

Deep Neural Network (DNN) is the very basic deep learning model. It contains multiple layers with many hidden neurons with non-linear activation function in each layer.  
Figure \ref{fig:dnn} shows one simple example of Deep Neural Network (DNN) model. This Deep neural network has one input layer, two hidden layers and one output layer. 
Training process of Deep Neural Network (DNN) includes forward propagation and back propagation. Forward propagation uses the current connection weight to give a prediction based on current state of model. Back propagation computes the amount of weight should be changed based on the difference of ground truth label and forward propagation prediction. Back propagation in Deep Neural Network (DNN) is a non-convex problem. Different initialization affects classification accuracy and convergence speed of models. 

Several unsupervised pretraining methods for neural network have been proposed to improve the performance of random initialized DNN, such as using stacks of RBMs (Restricted Boltzmann Machines) \cite{hinton2006reducing}, autoencoders \cite{vincent2010stacked}, or DBM (Deep Boltzmann Machines) \cite{salakhutdinov2009deep}. Compared to random initialization, pretraining followed with finetuning backpropagation will improve the performance significantly. 
Deep Belief Network (DBN) is a generative unsupervised pretraining network which uses stacked RBMs \cite{hinton2002training} during pretraining. A DNN with a corresponding configured DBN often produces much better results. DBN has undirected connections between its first two layers and directed connections between all its lower layers\cite{deng2012three} \cite{salakhutdinov2009deep}.   

Convolution Neural Network (CNN) \cite{lecun1998gradient} \cite{krizhevsky2012imagenet} \cite{lecun1995convolutional} has been proposed to deal with images, speech and time-series. This is because standard DNN has some limitations. Firstly, images, speeches are usually large. A simple Neural Network to process an image size of $100 \times 100$ with 1 layer of 100 hidden neurons will require 1,000,000 ($100 \times 100 \times 100$) weight parameters. With so many variables, it will lead to overfitting easily. Computation of standard DNN model requires expensive memory too. Secondly, standard DNN does not consider the local structure and topology of the input. For example, images have strong 2D local structure. Many areas in the image are similar. Speeches have a strong 1D structure, where variables temporally nearby are highly correlated. CNN forces the extraction of local features by restricting the receptive fields of hidden neurons to be local \cite{lecun1995convolutional}.

However, the training process for deep learning algorithms, including DNN, DBN, CNN, is computationally expensive. This is due to the large number of training data and a large number of parameters for multiple layers. 
Inspired from the \emph{shrinking} technique \cite{joachims1999making} \cite{narasimhan2014fast} used in accelerating computation of Support Vector Machines (SVM) algorithm and screening \cite{wang2013lasso} \cite{bonnefoy2014dynamic} technique used in LASSO, we propose an accelerating algorithm shrinking Deep Learning with Recall (sDLr). The main contribution of sDLr is that it can reduce the running time significantly. Though there is a trade-off between classification improvement and speedup on training time, for some data sets, sDLr approach can even improve classification accuracy. It should be noted that the approach sDLr is a general model and a new way of thinking, which can be applied to both large data, large network and small data small network, both sequential and parallel implementations. We will study the impact of proposed accelerating approaches on DNN, DBN and CNN using 4 data sets from computer vision and high energy physics, biology science.

\begin{figure}[t]
  \centering
  \includegraphics[width=0.8\columnwidth]{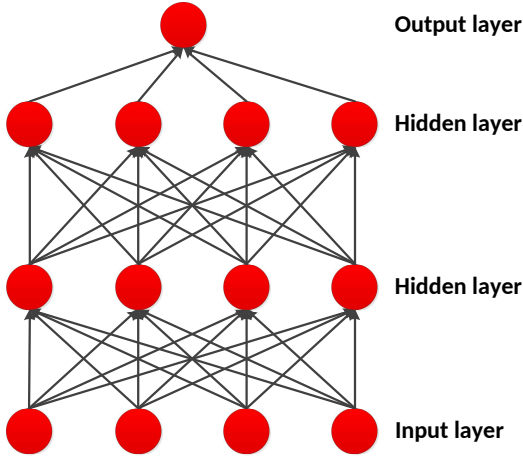}
  \caption{Deep Neural Network (DNN).}
  \label{fig:dnn}
\end{figure}

\section{Motivation}
The amount of data in our world has been exploding. Analyzing large data sets, so-called big data, will become a key basis of competition, underpinning new waves of productivity growth, innovation, and consumer interest \cite{manyika2011big}. A lot of big data technologies, including cloud computing, dimensionality reduction have been proposed \cite{zheng2015closed, zheng2014kernel, williams2014tidewatch, zhang2012virtual, zheng2011analysis}. Analyzing big data with machine learning algorithms requires special hardware implementations and large amount of running time. 

SVM \cite{suykens1999least} solves the following optimization problem:
\begin{align}
\label{eq:svm} 
&\min_{\ww,{\bf \xi},b} \frac{1}{2} \ww^T \ww + C \sum_{i=1}^n \xi_i,\\
\text{subject to } & y_i (\ww^T \Phi(\xx_i) - b) > 1 - \xi_i, \nonumber \\
& \xi_i >0, \; i = 1, ..., n, \nonumber
\end{align}%
where $\xx_i$ is a training sample, $y_i$ is the corresponding label, $\xi_i$ is positive slack variable, $\Phi(\xx_i)$ is mapping function, $\ww$ gives the solution and is known as weight vector, $C$ controls the relative importance of maximizing the margin and minimizing the amount of the slack. Since SVM learning problem has much less support vectors than training examples, shrinking \cite{joachims1999making} \cite{narasimhan2014fast} was proposed to eliminate training samples for large learning tasks where the fraction of support vectors is small compared to the training sample size or when many support vectors are at the upper bound of Lagrange multipliers.

LASSO \cite{tibshirani1996regression} is an optimization problem to find sparse representation of some signals with respect to a predefined dictionary. It solves the following problem:
\begin{align}
\label{eq:lasso} 
\min_{\xx} \frac{1}{2} \|D\xx - \yy \|^2_2 + \lambda \| \xx \|_1,
\end{align}%
where $\yy$ is a testing point, $D \in \Re^{p \times n}$ is a dictionary with dimension $p$ and size $n$, $\lambda$ is a parameter controls the sparsity of representation $\xx$. When both $p$ and $n$ are large, which is usually the case in practical applications, such as denoising or classification, it is difficult and time-intensive to compute. Screening \cite{wang2013lasso} \cite{bonnefoy2014dynamic} is a technique used to reduce the size of dictionary using some rules in order to accelerate the computation of LASSO.

Either in shrinking of SVM or in screening of LASSO, these approaches are trying to reduce the size of computation data. Inspired from these two techniques, we propose a faster and reliable approach for deep learning, shrinking Deep Learning. 

\section{Shrinking Deep Learning}
Given testing point $\xx_i \in \Re^{p \times 1}$, $i = 1, 2, ..., n$, let class indicator vector be $\yy_i^{(0)} \in \Re^{1 \times c}$, where $n$ is number of testing samples, $c$ is number of classes, $\yy_i$ has all $0$s except one $1$ to indicate the class of this test point. Let the output of a neural network for testing point $\xx_i$ be $\yy_i \in \Re^{1 \times c}$. $\yy_i$ contains continuous values and is the $i$th row of $\YY$. 

\subsection{Standard Deep Learning} Algorithm \ref{alg:dl} gives the framework of standard deep learning. During each epoch (iteration), standard deep learning first runs a forward-propagation on all training data, then computes the output $\YY(\ww)$, where output $\YY$ is a function of weight parameters $\ww$. Deep learning tries to find an optimal $\ww$ to minimize error loss $\ee = [e_1, ..., e_n]$, which can be sum squared error loss (DNN, DBN in our experiment) or softmax loss (CNN in our experiment). In backpropagation process, deep learning updates weight parameter vector using gradient descent. For an training data $\xx_i$, gradient descent can be denoted as:  
\begin{align}
\ww^{(epoch+1)} = \ww^{(epoch)} - \eta \nabla e_i( \ww^{(epoch)} ),
\label{eq:wupdate}
\end{align}%
where $\eta$ is step size. 

Before we present shrinking Deep Leaning algorithm, we first give Lemma \ref{lm:lm1}.

\begin{lemma}
Magnitude of gradient $\nabla e_i( \ww^{(epoch)})$ in Eq.(\ref{eq:wupdate}) is positive correlated with the error $e_i, \; i =1, ...,n$.
\label{lm:lm1}
\end{lemma}
\begin{proof}
In the case of sum squared error, error loss of sample $\xx_i$ is given as:
\begin{align}
e_i = \frac{1}{2} \sum_{k=1}^c (y_{ik}(\ww) - y_{ik}^{(0)} )^2, \; i =1, ...,n. \label{eq:err}
\end{align}%
Using Eq.(\ref{eq:err}), gradient $\nabla e_i(\ww)$ is:
\begin{align}
\nabla e_i(\ww) &= \sum_{k=1}^c (y_{ik}(\ww) - y_{ik}^{(0)} ) \nabla y_{ik}(\ww) \nonumber \\
\nabla e_i(\ww) &= \nabla \yy_{i}(\ww) (\yy_{i} - \yy_{i}^{(0)})^T. \label{eq:errProp}
\end{align}%
As we can see from Eq.(\ref{eq:errProp}), $\nabla e_i(\ww)$ is linear related to $(\yy_{i} - \yy_{i}^{(0)})$. Data points with larger error will have larger gradient, thus will have a stronger and larger correction signal when updating $\ww$. Data points with smaller error will have smaller gradient, thus will have a weaker and smaller correction signal when updating $\ww$.

In the case of softmax loss function, $e_i$ is denoted as:
\begin{align}
e_i &= - \sum_{k=1}^c y_{ik}^{(0)} log p_{ik}, \; i =1, ...,n  \label{eq:errsoftmax}\\
p_{ik} &= \frac{exp(y_{ik}(\ww))}{\sum_{j=1}^c exp(y_{ij}(\ww))}. 
\end{align}%
Using Eq.(\ref{eq:errsoftmax}), gradient $\nabla e_i(\ww)$ is:
\begin{align}
\nabla e_i(\ww) &= \sum_{k=1}^c \frac{\partial e_i(y_{ik})}{\partial y_{ik}} \nabla y_{ik}(\ww), \nonumber \\
\nabla e_i(\ww) &= \sum_{k=1}^c (p_{ik} - y_{ik}^{(0)}) \nabla y_{ik}(\ww).
\label{eq:errPropSoftmax}
\end{align}%
Now let's see the relation between softmax  loss function (Eq.(\ref{eq:errsoftmax})) and its gradient with respect to weight parameter $\ww$ (Eq.(\ref{eq:errPropSoftmax})). For example, given point $i$ is in class $1$, so $y_{i1}^{(0)} = 1$ and $y_{ij}^{(0)} = 0,\; j\not=1$. When $p_{i1}$ is large, $p_{i1} \to 1$, softmax  loss function (Eq.(\ref{eq:errsoftmax})) is very small. For gradient of softmax loss function (Eq.(\ref{eq:errPropSoftmax})), when $k=1$, $(p_{i1} - 1)$ is close to $0$; when $k\not=1$, $(p_{i1} - 0)$ is also close to $0$. In summary, when softmax  loss function (Eq.(\ref{eq:errsoftmax})) is very small, its gradient (Eq.(\ref{eq:errPropSoftmax})) is also very small.
\end{proof}

\begin{algorithm}[t!]
\small
\caption{Deep Learning (DL)}
\label{alg:dl}
\begin{algorithmic}[1]
\Require Data matrix $\XX \in \Re^{p \times n}$, class matrix $\YY^{(0)} \in \Re^{n \times c}$
\Ensure Classification error 
\State Preprocessing training data
\State Active training data index $\mathbb{A} = \{1, 2, ..., n\}$
\For {$epoch =1,2,...$}
	\State Run forward-propagation on $\mathbb{A}$
	\State Compute forward-propagation output $\YY \in \Re^{n \times c}$
	\State Run back-propagation
	\State Update weight $\ww$ using Eq.(\ref{eq:wupdate})
\EndFor
\State Compute classification error using $\YY$ and $\YY^{(0)}$
\end{algorithmic}
\end{algorithm}

\begin{algorithm}[t!]
\small
\caption{Shrinking Deep Learning (sDL)}
\label{alg:sdl}
\begin{algorithmic}[1]
\Require Data matrix $\XX \in \Re^{p \times n}$, class matrix $\YY^{(0)} \in \Re^{n \times c}$, elimination rate $s$ ($s$ is a percentage), stop threshold $t$
\Ensure Classification error 
\State Preprocessing training data
\State Active training data index $\mathbb{A} = \{1, 2, ..., n\}$
\For {$epoch =1,2,...$}
	\State Run forward-propagation on $\mathbb{A}$
	\State Compute forward-propagation output $\YY \in \Re^{n \times c}$
	\State Run back-propagation
	\State Update weight $\ww$ using Eq.(\ref{eq:wupdate})
	\If {$n_{epoch} >= t$}
		\State Compute error using Eq.(\ref{eq:err}) 
		\State Compute set $\mathbb{S}$, which contains indexes of $n_{epoch}s$  smallest $e_i$ values ($n_{epoch}$ is size of $\mathbb{A}$ in current epoch) 
		\State Eliminate all samples in $\mathbb{S}$ and update $\mathbb{A}$, $\mathbb{A} = \mathbb{A} - \mathbb{S}$ 
	\EndIf
\EndFor
\State Compute classification error using $\YY$ and $\YY^{(0)}$
\end{algorithmic}
\end{algorithm}

\begin{algorithm}[t!]
\small
\caption{Shrinking Deep Learning with Recall (sDLr)}
\label{alg:sdlr}
\begin{algorithmic}[1]
\Require Data matrix $\XX \in \Re^{p \times n}$, class matrix $\YY^{(0)} \in \Re^{n \times c}$, elimination rate $s$ ($s$ is a percentage), stop threshold $t$
\Ensure Classification error 
\State Preprocessing training data
\State Active training data index $\mathbb{A} = \{1, 2, ..., n\}$, $\mathbb{A}_0 = \mathbb{A}$
\For {$epoch =1,2,...$}
	\State Run forward-propagation on $\mathbb{A}$
	\State Compute forward-propagation output $\YY \in \Re^{n \times c}$
	\State Run back-propagation
	\State Update weight $\ww$ using Eq.(\ref{eq:wupdate})
	\If {$n_{epoch} >= t$}
		\State Compute error using Eq.(\ref{eq:err}) 
		\State Compute set $\mathbb{S}$, which contains indexes of $n_{epoch}s$  smallest $e_i$ values ($n_{epoch}$ is size of $\mathbb{A}$ in current epoch) 
		\State Eliminate all samples in $\mathbb{S}$ and update $\mathbb{A}$, $\mathbb{A} = \mathbb{A} - \mathbb{S}$ 
	\Else 
	\State Use all data for training, $\mathbb{A} = \mathbb{A}_0$
	\EndIf
\EndFor
\State Compute classification error using $\YY$ and $\YY^{(0)}$
\end{algorithmic}
\end{algorithm}

\subsection{Shrinking Deep Learning} In order to accelerate computation and inspired from techniques of shrinking in SVM and screening of LASSO, we propose shrinking Deep Learning in Algorithm \ref{alg:sdl} by eliminating samples with small error (Eq.(\ref{eq:err})) from training data and use less data for training. 

Algorithm \ref{alg:sdl} gives the outline of shrinking Deep Learning (sDL). Compared to standard deep learning in Algorithm \ref{alg:dl}, sDL requires two more inputs, elimination rate $s$ and stop threshold $t$. $s$ is a percentage indicating the amount of training data to be eliminated during one epoch, $t$ is a number indication to stop eliminating training data when $n_{epoch} < t$, where $n_{epoch}$ is current number of training data. We maintain an index vector $\mathbb{A}$. In Algorithm \ref{alg:dl}, both forward and backward propagation apply on all training data. In Algorithm \ref{alg:sdl}, the training process is applied on a subset of all training data. In the first epoch, we set $\mathbb{A} = \{1, 2, ..., n \}$ to include all training indexes. After forward and backward propagation in each epoch, we select the $n_{epoch}s$ indexes of training data with smallest error $e_i$, where $n_{epoch}$ is size of current number of training data $\mathbb{A}$. Then we eliminate indexes in $\mathbb{S}$ from $\mathbb{A}$, and update $\mathbb{A}$, $\mathbb{A} = \mathbb{A} -  \mathbb{S}$. When $n_{epoch} < t$, we stop eliminating training data anymore. Lemma \ref{lm:lm1} gives theoretical foundation that samples with small error will smaller impact on the gradient. Thus eliminating those samples will not impact the gradient significantly.  
Figure \ref{fig:mnist_ex} shows that the errors using sDL is smaller than errors using DL, which proves that sDL gives a stronger correction signal and reduce the errors faster.

When eliminating samples, elimination rate $s$ denotes the percentage of samples to be removed. We select the $n_{epoch}s$ indexes of training data with smallest error $e_i$. For the same epoch, in different batches, the threshold used to eliminate samples is different. Assume there are $n_{batch}$ batches one epoch, in every batch, we need to drop $n_{epoch}s/n_{batch}$ samples on average. In batch $i$, let the threshold to drop $n_{epoch}s/n_{batch}$ smallest error be $t_i$; in batch $i+1$, let the threshold be $t_{i+1}$. $t_i$ and $t_{i+1}$ will differ a lot. We use exponential smoothing \cite{gardner1985exponential} to adjust the threshold used in batch $i+1$: instead of using $t_{i+1}$ as the threshold to eliminate samples, we use the following $t'_{i+1}$:
\begin{align}
\label{eq:expsm} 
t'_{i+1} = \alpha t'_i + (1-\alpha) t_{i+1},
\end{align}
where $\alpha \in [0, 1)$ is a weight parameter which controls the importance of past threshold values, $t'_1 = t_1$. The intuition using exponential smoothing is that we want the threshold used in each epoch to be consistent. Samples with errors less than $t'_{i+1}$ in batch $i+1$ will be eliminated. If $\alpha$ is close to 0, the smoothing effect on threshold is not obvious; if $\alpha$ is close to 1, the threshold $t'_{i+1}$ will deviate a lot from $t_{i+1}$.  In practical, we find $\alpha$ between $0.5$ and $0.6$ is a good setting in terms of smoothing threshold. We will show this in experiment part.

\begin{figure}[t!]
  \centering
  \includegraphics[width=0.8\columnwidth]{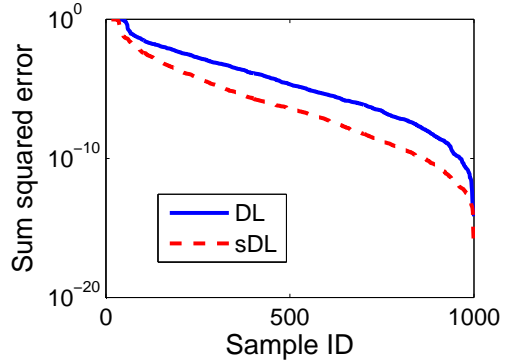}
  \caption{Sum squared errors (Eq.(\ref{eq:err})) using sDNN (shrinking DNN) is smaller than errors using standard DNN in the same epoch on 1000 samples from MNIST data.}
  \label{fig:mnist_ex}
\end{figure}

\section{Shrinking with Recall}
As the training data in sDL becomes less and less, the weight parameter $\ww$ trained is based on the subset of training data. It is not optimized for the entire training dataset. We now introduce shrinking Deep Learning with recall (Algorithm \ref{alg:sdlr}) to deal with this situation. In order to utilize all the training data, when the number of active training samples $n_{epoch} < t$, we start to use all training samples, as shown in Algorithm \ref{alg:sdlr}, $\mathbb{A} = \mathbb{A}_0$. Algorithm \ref{alg:sdlr} ensures that the model trained is optimized for the entire training data.
Shrinking with recall of Algorithm \ref{alg:sdlr} will produce competitive classification performance with standard Deep Learning of Algorithm \ref{alg:dl}. 
In experiment, we will also investigate the impact the threshold $t$ on the classification results (see Figure \ref{fig:sdnnr_t}). 

\begin{figure}[h!]
\centering
  \begin{subfigure}{0.9\columnwidth}
  \centering
  \includegraphics[trim=0.5cm 0.5cm 0.5cm 0.5cm, clip=true, width=\textwidth]{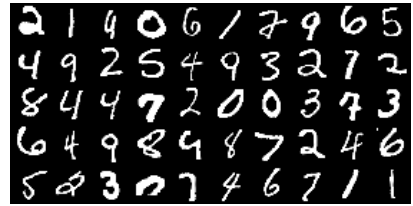}
  \caption{MNIST (10 classes, size $28\times 28$, randomly select 50 images).}
  \label{fig:mnist}
  \end{subfigure}
    \begin{subfigure}{0.9\columnwidth}
  \centering
  \includegraphics[trim=0.9cm 1cm 0.9cm 0.1cm, clip=true,width=0.98\textwidth]{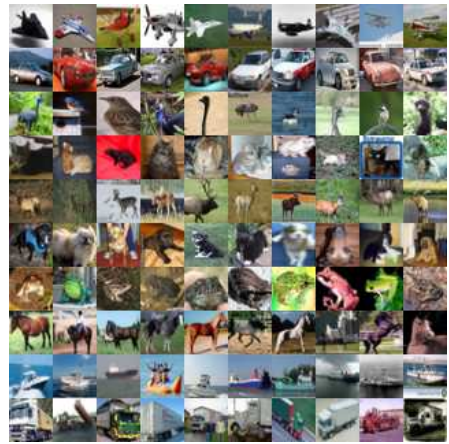}
  \caption{CIFAR-10 (10 classes in total, size $32\times 32$, each row is a class, randomly select 10 images from each class).}
  \label{fig:cifar10}
  \end{subfigure}
  \caption{Sample images.}
  \label{fig:sample_img}
\end{figure}

\begin{figure}[t!]
\centering
  \begin{subfigure}{0.65\columnwidth}
  \centering
  \includegraphics[width=\textwidth]{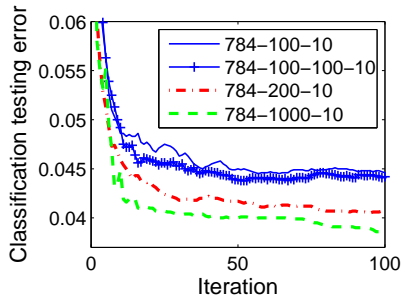}
  \caption{Testing error.}
  \label{fig:mnist_network}
  \end{subfigure}
    \begin{subfigure}{0.65\columnwidth}
  \centering
  \includegraphics[width=\textwidth]{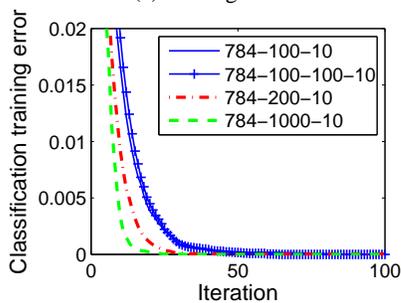}
  \caption{Training error.}
  \label{fig:mnist_network_training}
  \end{subfigure}
  \caption{MNIST DNN testing and training error on different network (100 iterations/epochs).}
  \label{fig:mnist_network0}
\end{figure}

\begin{figure}[h!]
\centering
  \begin{subfigure}{0.65\columnwidth}
  \centering
  \includegraphics[width=\textwidth]{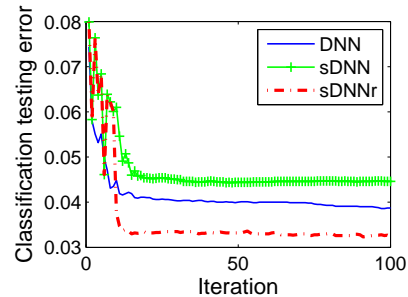}
  \caption{Testing error.}
  \label{fig:mnist_test}
  \end{subfigure}
    \begin{subfigure}{0.65\columnwidth}
  \centering
  \includegraphics[width=\textwidth]{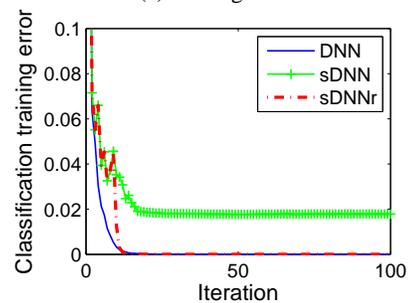}
  \caption{Training error.}
  \label{fig:mnist_train}
  \end{subfigure}
  \caption{MNIST testing and training error (100 iterations/epochs).}
  \label{fig:mnist_result}
\end{figure}

\begin{figure}[h!]
\centering
  \begin{subfigure}{0.65\columnwidth}
  \centering
  \includegraphics[width=\textwidth]{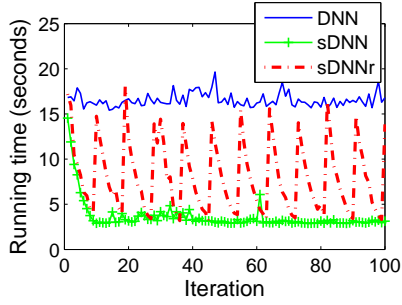}
  \caption{Training time.}
  \label{fig:mnist_time}
  \end{subfigure}
    \begin{subfigure}{0.65\columnwidth}
  \centering
  \includegraphics[width=\textwidth]{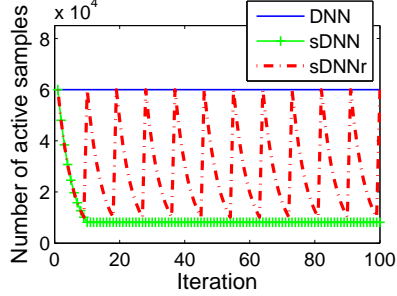}
  \caption{Number of active samples.}
  \label{fig:mnist_num}
  \end{subfigure}
  \caption{MNIST training time and number of active samples (100 iterations/epochs).}
  \label{fig:mnist_timenum}
\end{figure}

\begin{figure}[h!]
  \centering
  \includegraphics[width=0.65\columnwidth]{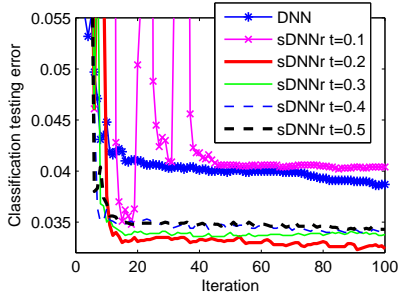}
  \caption{MNIST classification testing error using different recall threshold $t$ (100 iterations/epochs).}
  \label{fig:sdnnr_t}
\end{figure}

\begin{table}[t!]
\centering
\caption{Overview of data sets used in this paper.}
\label{tab:data}
\resizebox{1\columnwidth}{!}{
\begin{tabular}{c|ccccc}
\hline\hline
Dataset  & Dimensionality  & Training Set  & Testing Set   \\
\hline
MNIST   & 784 ($28 \times 28$ grayscale) & 60K & 10K \\
CIFAR-10 & 3072 ($32 \times 32$ color)  & 50K & 10K  \\
Higgs Boson  & 7 & 50K & 20K \\
Alternative Splicing & 3446   & 2500 & 945 \\
\hline\hline
\end{tabular}
}
\end{table}

\begin{table*}[ht!]
\centering
\caption{MNIST classification error improvement (IMP) and training time Speedup.}
\label{tab:mnistresult}
\resizebox{1\textwidth}{!}{
\begin{tabular}{c|ccc|ccc|ccc}
\hline\hline
Method                       &  DNN & sDNNr & IMP/Speedup & DBN & sDBNr &IMP/Speedup &  CNN & sCNNr & IMP/Speedup \\
\hline
Testing error & 0.0387 & 0.0324 & $16.3\%$    & 0.0192 &0.0182& $5.21\%$ &0.0072 & 0.0073 &$-1.39\%$ \\
Training time (s)            & 1653  & 805     & $2.05$    & 1627   &700   & $2.32$  & 3042   & 1431 & $2.13$ \\
\hline\hline
\end{tabular}}
\end{table*}

\section{Experiments}
In experiment, we test our algorithms on data sets of different domains using 5 different random initialization. The data sets we used are listed in Table \ref{tab:data}. MNIST is a standard toy data set of handwritten digits; CIFAR-10 
contains tiny natural images; Higgs Boson is a dataset from high energy physics.
Alternative Splicing is RNA features used for predicting alternative gene splicing. 
We use DNN and DBN implementation from \cite{palm2012prediction} and CNN implementation from \cite{vedaldi2014matconvnet}. All experiments were conducted on a laptop with Intel Core i5-3210M CPU 2.50GHz, 4GB RAM, Windows 7 64-bit OS.

\subsection{Results on MNIST}
MNIST is a standard toy data set of handwritten digits containing 10 classes. It contains 60K training samples and 10K testing samples. The image size is $784$ (grayscale $28 \times 28$). Figure \ref{fig:mnist} shows some examples of MNIST dataset.

\subsubsection{Deep Neural Network} In experiment, we first test on some network architecture and find a better one for our further investigations. 
Figure \ref{fig:mnist_network} and Figure \ref{fig:mnist_network_training} show the testing and training classification error for different network settings. 
Results show that ``$784\mbox{-}1000\mbox{-}10$" is a better setting with lower testing error and converges faster in training. We will use network ``$784\mbox{-}1000\mbox{-}10$" for DNN and DBN on MNIST. Learning rate is set to be 1; activation function is tangent function and output unit is sigmoid function. 

Figure \ref{fig:mnist_result} shows the testing error and training error of using standard DNN, sDNN (Shrinking DNN) and sDNNr (shrinking DNN with recall).  
Results show that sDNNr improves the accuracy of standard DNN. 
While for training error, both DNN and sDNNr give almost 0 training error.

Figure \ref{fig:mnist_timenum} shows training time and number of active samples in each iteration (epoch). In our experiments, for sDNN and sDNNr, we set eliminate rate $s = 20\%$. sDNNr has a recall process to use the the entire training samples, as shown in Figure \ref{fig:mnist_timenum}. When the number of active samples is less than $t = 20\% \times 60K$ of total training samples, we stop eliminating samples. The speedup using sDNNr compared to DNN is
\begin{align}
\label{eq:speedup} 
Speedup = \frac{t_{DNN}}{t_{sDNNr}} = 2.05.
\end{align}

Recall is a technique when the number of training samples is decreased to a threshold $t$, we start to use all training samples. There is a trade-off between speedup and classification error: setting a lower $t$ could reduce computation time more, but could increase classification error. Figure \ref{fig:sdnnr_t} shows the effect of using different recall threshold $t$ sDNNr on MNIST data. When we bring all training samples back at $t=20\% \times 60K$, we get the best testing error. It is worth noting that the classification error of sDNNr is improved compared to standard DNN, which could imply that there is less overfitting for this data set.

\begin{figure}[h!]
  \centering
  \includegraphics[width=0.8\columnwidth]{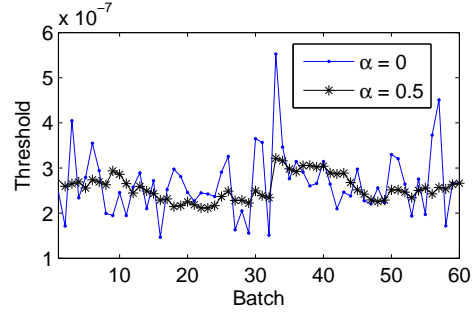}
  \caption{Exponential smoothing (see Eq.(\ref{eq:expsm})) effect on one epoch of MNIST (60 batches with 1000 samples/batch ).}
  \label{fig:expsm}
\end{figure}

Figure \ref{fig:expsm} shows an example of exponential smoothing on the elimination threshold (Eq.(\ref{eq:expsm})) during one epoch. The threshold using $\alpha=0.5$ smooths the curve a lot. 

\subsubsection{Deep Belief Network}
Figure \ref{fig:DBN} shows the classification testing error and training time of using Deep Belief Network (DBN) and shrinking DBN with recall (sDBNr) on MNIST. Network setting is same as it is in DNN experiment.
sDBNr further reduces the classification error of DBN to $0.0182$ by using sDBNr. 

\begin{figure}[t]
\centering
  \begin{subfigure}{0.65\columnwidth}
  \centering
  \includegraphics[width=\textwidth]{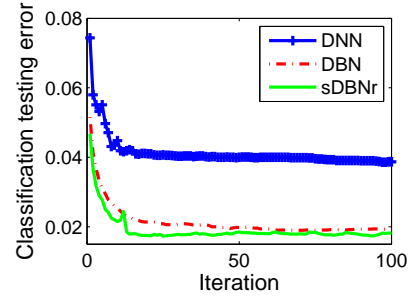}
  \caption{Testing error}
  \label{fig:dbnError}
  \end{subfigure}
  \begin{subfigure}{0.65\columnwidth}
  \centering
  \includegraphics[width=\textwidth]{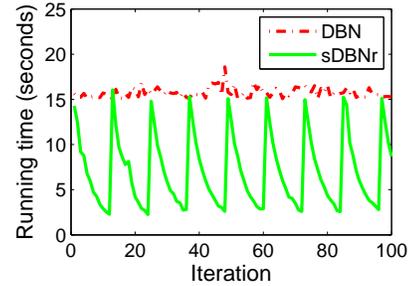}
  \caption{Training Time}
  \label{fig:dbnTime}
  \end{subfigure}
  \caption{MNIST DBN result(100 iterations/epochs): (a) compares classification testing error; (b) compares training time.}  
  \label{fig:DBN}
\end{figure}

\subsubsection{Convolution Neural Networks (CNN)}
The network architecture used in MNIST is 4 convolutional layers with each of the first 2 convolutional layers followed by a max-pooling layer, then 1 layer followed by a ReLU layer, 1 layer followed by a Softmax layer. The first 2 convolutional layers have $5 \times 5$ receptive field applied with a stride of 1 pixel. The 3rd convolutional layer has $4 \times 4$ receptive field and the 4th layer has $1 \times 1$ receptive field with a stride of 1 pixel. The max pooling layers pool $2 \times 2$ regions at strides of 2 pixels. Figure \ref{fig:cnn} shows the classification testing error and training time of CNN on MNIST data. 

\begin{figure}[t]
\centering
  \begin{subfigure}{0.65\columnwidth}
  \centering
  \includegraphics[width=\textwidth]{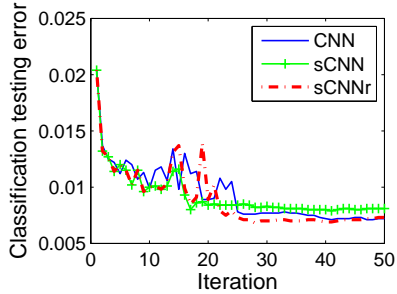}
  \caption{Testing error}
  \label{fig:cnnError}
  \end{subfigure}
  \begin{subfigure}{0.65\columnwidth}
  \centering
  \includegraphics[width=\textwidth]{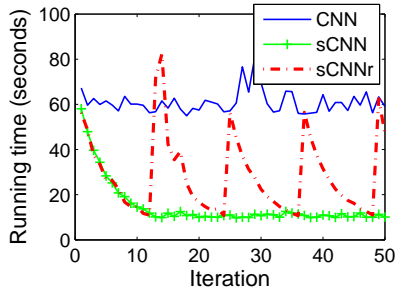}
  \caption{Training Time}
  \label{fig:cnnTime}
  \end{subfigure}
  \caption{MNIST CNN result(50 iterations/epochs): (a) compares classification testing error; (b) compares training time.}  
  \label{fig:cnn}
\end{figure}

Table \ref{tab:mnistresult} summarizes the classification error improvement (IMP) and training time speedup of DNN, DBN and CNN on MNIST data, where improvement is $IMP = (err_{DL}-err_{sDLr})/err_{DL}$.

\begin{table}[t!]
\centering
\caption{CIFAR-10 classification error improvement (IMP) and training time Speedup.}
\label{tab:c10}
\resizebox{\columnwidth}{!}{
\begin{tabular}{c|ccc}
\hline\hline
Method                       & CNN & sCNNr & IMP/Speedup \\
\hline
Testing error (top 1) & 0.2070 & 0.2066  & $0.19\%$ \\
Training time (s)            & 5571  & 3565     & $1.56$  \\
\hline\hline
\end{tabular}}
\end{table}

\subsection{Results on CIFAR-10}
CIFAR-10 \cite{krizhevsky2009learning} data contains 60,000 $32 \times 32$ color image in 10 classes, with 6,000 images per class. There are 50,000 training and 10,000 testing images. 
CIFAR-10 is an object dataset, which includes airplane, car, bird, cat and so on and classes are completely mutually exclusive. 
In our experiment, we use CNN network to evaluate the performance in terms of classification error. Network architecture uses 5 convolutional layers: for the first three layers, each convolutional layer is followed by a max pooling layer; $4$th convolutional layer is followed by a ReLU layer; the 5th layer is followed by a softmax loss output layer. Table \ref{tab:c10} shows the classification error and training time. Top-1 classification testing error in Table \ref{tab:c10} means that the predict label is determined by considering the class with maximum probability only.

\begin{table}[t!]
\centering
\caption{Higgs Boson classification error improvement (IMP) and training time Speedup.}
\label{tab:higgs}
\resizebox{\columnwidth}{!}{
\begin{tabular}{cc|ccc}
\hline\hline
DNN Network & Method       & DNN & sDNNr & IMP/Speedup \\
\hline
 $7\mbox{-}20\mbox{-}20\mbox{-}2$ &Testing error & 0.4759 & 0.4512 & $5.19\%$ \\
                & Training time (s) & 21  & 13     & $1.62$  \\
\hline
 $7\mbox{-}50\mbox{-}2$ &Testing error & 0.3485 & 0.3386 & $2.84\%$ \\
             & Training time (s)  & 52  & 18     & $2.89$  \\
\hline\hline
\end{tabular}}
\end{table}

\begin{table}[t]
\centering
\caption{Alternative Splicing error improvement (IMP) and training time Speedup.}
\label{tab:asdata}
\resizebox{\columnwidth}{!}{
\begin{tabular}{cc|ccc}
\hline\hline
DNN Network & Method       & DNN & sDNNr & IMP/Speedup \\
\hline
 $1389-100-3$ &Testing error & 0.2681  &  0.2960 & $10.4\% $ \\
                & Training time (s) &  32  &  20    & $1.60 $  \\
\hline\hline
\end{tabular}}
\end{table}

\subsection{Results on Higgs Boson}
Higgs Boson is a subset of data from \cite{baldi2014searching} with $50,000$ training and $20,000$ testing. Each sample is a signal process which either produces Higgs bosons particle or not. We use 7 high-level features derived by physicists to help discriminate particles between the two classes. Both activation function and output function were sigmoid function. The DNN batchsize is $100$ and recall threshold $t=20\% \times 50,000$. We test on different network settings and choose the best. Table \ref{tab:higgs} shows the experiment results using different network. 

\subsection{Results on Alternative Splicing}
Alternative Splicing \cite{xiong2011bayesian} is a set of RNA sequences used in bioinfomatics. It contains 3446 cassette-type mouse exons with 1389 features per exon.
We randomly select 2500 exons for training and use the rest for testing. 
For each exon, the dataset contains three real-valued positive prediction targets $y_i = [ q^{inc}\;\; q^{exc} \;\;q^{nc}]$, corresponding to probabilities that the exon is more likely to be included in the given tissue, more likely to be excluded, or more likely to exhibit no change relative to other tissues. To demonstrate the effective of proposed shrinking Deep Learning with recall approach, we use a simple DNN network of different number of layers and neurons with optimal tangent activation function and sigmoid output function. 
We use the following average sum squared error criteria to evaluate the model performance $error = \sum_{i=1}^{n} \| \yy_i - \yy_i^{(0)} \|^2 / n$,
where $\yy_i$ is the predict vector label and $\yy_i^{(0)}$ is the ground-truth label vector, $n$ is number of samples. The DNN batchsize is $100$ and recall threshold $t=20\% \times 2500$. We test on different network settings and choose the best. Table \ref{tab:asdata} shows the experiment result.  


\section{Conclusion}
In conclusion, we proposed a shrinking Deep Learning with recall (sDLr) approach and the main contribution of sDLr is that it can reduce the running time significantly. Extensive experiments on 4 datasets show that shrinking Deep Learning with recall can reduce training time significantly while still gives competitive classification performance.

\bibliographystyle{IEEEtran}
\bibliography{ref}

\end{document}